\documentclass{article}

\usepackage{graphicx} 
\usepackage{subfigure} 

\usepackage{natbib}

\usepackage{algorithm}
\usepackage{algorithmic}

\usepackage{hyperref}

\newtheorem{theorem}{Theorem}[section]
\newenvironment{proof}[1][Proof]{\begin{trivlist}
\item[\hskip \labelsep {\bfseries #1}]}{\end{trivlist}}

\newcommand{\qed}{\nobreak \ifvmode \relax \else
      \ifdim\lastskip<1.5em \hskip-\lastskip
      \hskip1.5em plus0em minus0.5em \fi \nobreak
      \vrule height0.75em width0.5em depth0.25em\fi}
      
 \newcommand{\reals}{\rm I\!R}

 \usepackage[accepted]{icml2012}

\icmltitlerunning{Planning by Prioritized Sweeping with Small Backups}

\begin{document} 

\twocolumn[
\icmltitle{Planning by Prioritized Sweeping with Small Backups}

\icmlauthor{Harm van Seijen}{harm.vanseijen@ualberta.ca}
\icmlauthor{Richard S. Sutton}{sutton@cs.ualberta.ca}
\icmladdress{Department of Computing Science, University of Alberta, Edmonton, Alberta, T6G 2E8, Canada}

\icmlkeywords{reinforcement learning}

\vskip 0.3in
]

\begin{abstract} 
Efficient planning plays a crucial role in model-based reinforcement learning. Traditionally, the main planning operation is a full backup based on the current estimates of the successor states. Consequently, its computation time is proportional to the number of successor states. In this paper, we introduce a new planning backup that  uses only the current value of a single successor state and has a computation time independent of the number of successor states. This new backup, which we call a \emph{small backup}, opens the door to a new class of model-based reinforcement learning methods that exhibit much finer control over their planning process than traditional methods. We empirically demonstrate that this increased flexibility allows for more efficient planning by showing that an implementation of prioritized sweeping based on small backups achieves a substantial performance improvement over classical implementations. 

\end{abstract} 

\section{Introduction}

In \emph{reinforcement learning}
(RL) \cite{kaelbling:jair96,sutton:book98}, an agent seeks an optimal
control policy for a sequential decision problem in an initially unknown environment.
The environment provides feedback on the agent's behavior in the form of a reward signal.
The agent's goal is to maximize the expected \emph{return}, which is the discounted sum of rewards over future timesteps.
An important performance measure in RL is the \emph{sample efficiency}, which refers to the number of environment interactions that is required to obtain a good policy.

Many solution strategies improve the policy by iteratively improving a \emph{state-value} or \emph{action-value function}, which provide estimates of the expected return under a given policy for (environment) states or state-action pairs, respectively. Different approaches for updating these value functions exist. In terms of sample efficiency, one of the most effective approaches is to estimate the environment model using observed samples and to compute, at each time step, the \mbox{(action-)}value function that is optimal with respect to the model estimate using planning techniques. A popular planning technique used for this is \emph{value iteration} (VI) \cite{sutton:ml1988, watkins89}, which performs sweeps of backups through the state or state-action space, until the (action-)value function has converged.
 
A drawback of using  VI is that it is computationally very expensive, making it infeasible for many practical applications. Fortunately, efficient approximations can be obtained by limiting the number of backups that is performed per timestep. A very effective approximation strategy is \emph{prioritized sweeping} \cite{moore:mlj93,peng:ab1993}, which prioritizes backups that are expected to cause large value changes. This paper introduces  a new backup that enables a dramatic improvement in the efficiency of prioritized sweeping.

The main idea behind this new backup is as following. Consider that we are interested in some estimate $A$ that is constructed from a sum of other estimates $X_i$. The estimate $A$ can be computed using a \emph{full backup}:
\begin{displaymath}
A \leftarrow \sum_i X_i\thinspace.
\end{displaymath}
If the estimates $X_i$ are updated, $A$ can be recomputed by redoing the above backup. Alternatively, if we know that only $X_j$ received a significant value change, we might want to update $A$ for only $X_j$. Let us indicate the old value of $X_j$, used to construct the current value of $A$, as $x_j$. $A$ can then be updated by subtracting this old value and adding the new value:
\begin{displaymath}
A \leftarrow   A - x_j +  X_j\thinspace.
\end{displaymath}

This kind of backup, which we call a \emph{small backup}, is computationally cheaper than the full backup. The trade-off is that, in general, more memory is required for storing the estimates $x_i$ associated with $A$. In planning, where the $X$ estimates correspond to state-value estimates and $A$ corresponds to a state or state-action estimate, this is not a serious restriction, because a full model is stored already. The additional memory required has the same order of complexity as the memory required for storage of the model.

The core advantage of small backups over full backups is that they enable finer control over the planning process. This allows for more effective update strategies, resulting in improved trade-offs between computation time and quality of approximation of the VI solution (and hence sample efficiency). We demonstrate this empirically by showing that a prioritized sweeping implementation based on small backups yields a substantial performance improvement over the two classical implementations \citep{moore:mlj93, peng:ab1993}.

In addition, we demonstrate the relevance of small backups in domains with severe constraints on computation time, by showing that a method that performs one small backup per time step has an equal computation time complexity as TD(0), the classical method that performs one \emph{sample backup} per timestep. Since sample backups introduce sampling variance, they require a step-size parameter to be tuned for optimal performance. Small backups, on the other hand, do not introduce sampling variance, allowing for a parameter-free implementation. We empirically demonstrate that the performance of a method that performs one small backup per time step is similar to the optimal performance of TD(0), achieved by carefully tuning the step-size parameter.

\section{Reinforcement Learning Framework}

RL problems are often formalized as \emph{Markov
  decision processes} (MDPs), which can be described as tuples
$\langle \mathcal{S}, \mathcal{A}, \mathcal{P}, \mathcal{R}, \gamma \rangle$
consisting of $\mathcal{S}$, the set of all states; $\mathcal{A}$, the
set of all actions; $\mathcal{P}_{sa}^{s'} = Pr(s^\prime | s,a)$,
the transition probability from state $s \in \mathcal{S}$ to state $s^\prime$ when action $ a \in \mathcal{A}$ is taken; $\mathcal{R}_{sa} = E\{r|s,a\}$, the reward function giving the expected reward $r$ when action $a$ is taken in state $s$; and $\gamma$, the discount factor controlling the weight of future rewards versus that of the immediate reward.

Actions are selected at discrete timesteps $t = 0,1,2,...$ according to a \emph{policy} $\pi: \mathcal{S} \times \mathcal{A} \rightarrow [0,1]$, which defines for each action the selection probability conditioned on the state. In general, the goal of RL is to improve the policy in order to increase the \emph{return} $G$, which is the discounted cumulative reward
\begin{displaymath}
G_t = r_{t+1} + \gamma\,r_{t+2} + \gamma^2\,r_{t+3}+... = \sum_{k=1}^\infty\,\gamma^{k-1}\, r_{t+k}\thinspace,
\end{displaymath}
where $r_{t+1}$ is the reward received after taking action $a_t$ in state $s_t$ at timestep $t$. 

The \emph{prediction task} consists of determining the \emph{value function} $V^{\pi}(s)$, which gives the expected return when policy $\pi$ is followed, starting from state $s$.  $V^{\pi}(s)$ can be found by making use of the \emph{Bellman equations} for state values, which state the following:
\begin{equation}
V^{\pi}(s) = \mathcal{R}_s + \gamma\sum_{s'}\mathcal{P}_{s}^{s'} V^{\pi}(s')\thinspace,\label{eq:bellman_state}
\end{equation}
where $\mathcal{R}_s \! = \!\!  \sum_{a} \pi(s,a) \mathcal{R}_{sa}$ and $\mathcal{P}_{s}^{s'}\! = \!\!\sum_{a} \pi(s,a) \mathcal{P}_{sa}^{s'}$.

\emph{Model-based methods} use samples to update estimates of the transition probabilities, $\hat \mathcal{P}_{s}^{s'}$, and reward function, $\hat \mathcal{R}_{s}$. With these estimates, they can iteratively improve an estimate $V$ of $V^{\pi}$, by performing \emph{full backups}, derived from Equation (\ref{eq:bellman_state}):
\begin{equation}
V(s) \leftarrow \hat\mathcal{R}_s + \gamma\sum_{s'}\hat\mathcal{P}_{s}^{s'} V(s')\thinspace.
\label{eq:full value backup}
\end{equation}

In the \emph{control task}, methods often aim to find the optimal policy $\pi^*$, which maximizes the expected return. This policy is the greedy policy with respect to the optimal \emph{action-value function} $Q^*(s,a)$, which gives the expected return when taking action $a$ in state $s$, and following $\pi^*$ thereafter.  
This function is the solution to the \emph{Bellman optimality equation}  for action-values:
\begin{equation}
Q^*(s,a) = \mathcal{R}_{sa} + \gamma\sum_{s'}\mathcal{P}_{sa}^{s'} \max_{a'} Q^*(s',a')\thinspace. \label{eq:belllman_action}
\end{equation}
The optimal value function is related to the optimal action-value function through: $V^*(s) = \max_a Q^*(s,a)$.

Model-based methods  can iteratively improve estimates $Q$ of $Q^*$ by performing full backups derived from Equation (\ref{eq:belllman_action}):
\begin{equation}
Q(s,a) \leftarrow \hat \mathcal{R}_{sa} + \gamma\sum_{s'}\hat \mathcal{P}_{sa}^{s'} \max_{a'} Q(s',a')\thinspace,\label{eq:full action-value backup}
\end{equation}
where $\hat \mathcal{R}_{sa}$ and $\hat \mathcal{P}_{sa}^{s'} $ are estimates of $\mathcal{R}_{sa}$ and $\mathcal{P}_{sa}^{s'}$, respectively. 

\emph{Model-free methods} do not maintain an model estimate, but update a value function directly from samples. A classical example of a \emph{sample backup}, based on sample $(s,r,s')$ is the TD(0) backup:
\begin{equation}
V(s) \leftarrow V(s) + \alpha \,(r + \gamma V(s') - V(s))\thinspace,
\end{equation}
where $\alpha$ is the step-size parameter. 

\section{Small Backup}

This section introduces the small backup. We start with small state-value backups for the prediction task. Section \ref{sec:action-value backups} discusses small action-value backups for the control task.

\subsection{Value Backups}

In this section, we introduce a small backup version of the full  backup for prediction (backup \ref{eq:full value backup}). In the introduction, we showed that a small backup requires storage of the component values that make up the current value of a variable. In the case of a small value backup, 
the component values correspond to the values of successor states. We indicate these values by the function $U_s:  \mathcal{S} \times \mathcal{S} \rightarrow \reals$. So, $U_s(s')$ is the value estimate of $s'$ associated with $s$.

Using $U_s$, $V(s)$ can be updated with only the current value of a single successor state, $s'$, as demonstrated by the following theorem. The three backups shown in the theorem form together the small backup.
\begin{theorem}
If the current relation between $V(s)$ and $U_s$ is given by
\begin{equation}
V(s) = \hat \mathcal{R}_{s} + \gamma\sum_{s''}\hat \mathcal{P}_{s}^{s''} U_s(s'')\thinspace, \label{eq:V_Vs_relation}
\end{equation}
then, after performing the following backups:
\begin{eqnarray}
tmp &\leftarrow& V(s')  \label{eq:small_1}\\
V(s) &\leftarrow& V( s) +  \gamma \mathcal{P}_{s}^{s'} [V(s') - U_{s}(s')] \quad \label{eq:small_2}\\
U_{s}(s') &\leftarrow& tmp \label{eq:small_3}\thinspace,
\end{eqnarray}
relation (\ref{eq:V_Vs_relation}) still holds, but $U_{s}(s')$ is updated to $V(s')$.\label{th:value theorem}
\end{theorem}

\begin{proof}
Backup (\ref{eq:small_2}) subtract the component in relation (\ref{eq:V_Vs_relation}) corresponding to $s'$ from $V(s)$ and adds a new component based on the current value estimate of $s'$:
\begin{displaymath}
V(s) \leftarrow  V(s) - \gamma\hat \mathcal{P}_{s}^{s'} U_s(s') + \gamma\hat \mathcal{P}_{s}^{s'} V(s')\thinspace.
\end{displaymath}
Hence, relation (\ref{eq:V_Vs_relation}) is maintained, while $U_s(s')$ is updated.
Note that $V(s')$ needs to be stored in a temporary variable, since backup (\ref{eq:small_2}) can alter the value of $V(s')$ if $s' = s$. 
\qed
\end{proof}

\subsection{Value Correction after Model Update}

Theorem \ref{th:value theorem} relies on relation (\ref{eq:V_Vs_relation}) to hold. If the model gets updated, this relation now longer holds. In this section, we discuss how to restore relation (\ref{eq:V_Vs_relation}) in a computation-efficient way for the commonly used model estimate:
\begin{eqnarray}\
\hat\mathcal{P}_{s}^{s'} &\leftarrow& N_{s}^{s'}/N_{s}  \label{model1}\\
\hat{\mathcal{R}}_{s} &\leftarrow& R^{sum}_{s} /N_{s}  \thinspace, \label{model2}
\end{eqnarray}
where $N_{s}$ counts the number of times state $s$ is visited, $N_{s}^{s'}$ counts the number of times $s'$ is observed as successor state of $s$, and $R^{sum}_{s}$ is the sum of observed rewards for $s$.

\begin{theorem}
If currently, the following relation holds:
\begin{displaymath}
V(s) = \hat \mathcal{R}_{s} + \gamma\sum_{s''}\hat \mathcal{P}_{s}^{s''} U_s(s'')\thinspace,
\end{displaymath}
and a sample $(s,r,s')$ is observed, then, after performing the backups:
\begin{eqnarray}
&&N_s \leftarrow N_s + 1;\quad N_{s}^{s'} \leftarrow N_{s}^{s'} + 1\quad\\
&&V(s) \leftarrow \Big[V(s)(N_{s}-1) + r + \gamma U_s(s')\Big]/N_{s}\thinspace. \quad\quad \label{eq:value correction}
\end{eqnarray}
the relation still holds, but with updated values for $\hat \mathcal{R}_{s}$ and $\hat \mathcal{P}_{s}^{s''}$.
\end{theorem}

\begin{proof}[Proof (sketch)]
Backup (\ref{eq:value correction}) updates $V(s)$ by computing a weighted average of $V(s)$  and  $r + \gamma U_s(s')$. The value change this causes is the same as the value change caused by updating the model and then performing a full backup of $s$ based on $U_s$. \qed
\end{proof}

Algorithm \ref{al:small backup evaluation} shows pseudo-code for a general class of prediction methods based on small backups. Surpisngly, while it is a planning method, $\hat \mathcal{R}_s$ is never explicitly computed, saving time and memory. Note that the computation per time step is fully independent of the number of successor states.  Members of this class need to specify the number of iterations (line 8) as well as a strategy for selecting state-successor pairs (line 9). 

\begin{algorithm}[thb]
\begin{algorithmic}[1]
\STATE initialize $V(s)$ arbitrarily for all $s$
\STATE initialize $U_{s}(s')  = V(s')$ for all $s,s'$
\STATE initialize $N_{s}, N_{s}^{s'}$ to 0 for all $s,s'$
\LOOP[over timesteps]
\STATE observe transition $(s,r,s')$
\STATE $N_{s} \leftarrow N_{s} + 1; \quad N_{s}^{s'} \leftarrow N_{s}^{s'} + 1$
\STATE $V(s) \leftarrow \Big[V(s)(N_{s}-1) + r + \gamma \,U_s(s')\Big]/N_{s}$
\LOOP[for a number of iterations]
\STATE select a pair $(\bar s, \bar s')$ with $N_{\bar s}^{ \bar s'} > 0$
\STATE $tmp \leftarrow V(\bar s')$
\STATE $V(\bar s) \leftarrow V(\bar s) +  \gamma N_{\bar s}^{\bar s'}/N_{\bar s}\cdot [V(\bar s') - U_{\bar s}(\bar s')]$
\STATE $U_{\bar s}(\bar s') \leftarrow tmp$
\ENDLOOP
\ENDLOOP
\caption{Prediction with Small Backups}
\label{al:small backup evaluation}
\end{algorithmic}
\end{algorithm}

\subsection{Action-value Backups}
\label{sec:action-value backups}

Before we can discuss small action-value backups, we have to discuss a more efficient implementation of the full action-value backup. Backup (\ref{eq:full action-value backup}) has a computation time complexity of  $\mathcal{O}(|\mathcal{S}||\mathcal{A}|)$. A more efficient implementation can be obtained by storing state-values, besides action-values, according to $V(s) = \max_a Q(s,a)$. 
Backup (\ref{eq:full action-value backup}) can then be implemented by:
\begin{eqnarray}
Q(s,a) &\leftarrow& \hat \mathcal{R}_{sa}+ \gamma \sum_{s'} V(s')  \label{eq:QV}\\
V(s)  &\leftarrow& \max_{a'} Q(s,a) \label{eq:VQ}\thinspace.
\end{eqnarray}
The combined computation time of these backups is $\mathcal{O}(|\mathcal{S}| + |\mathcal{A}|)$, a considerable reduction.

Backup (\ref{eq:QV}) is similar in form as the prediction backup. Hence, we can make a small backup version of it similar to the one in the prediction case. 
The theorems below are the control versions of the theorems for the prediction case. They can be proven in a similar way as the prediction theorems. 

\begin{theorem}
If the current relation between $Q(s,a)$ and $U_{sa}$ is given by
\begin{equation}
Q(s,a) \leftarrow \hat \mathcal{R}_{sa} + \gamma \mathcal{P}_{sa}^{s''}\sum_{s''} U_{sa}(s'')\thinspace,  \label{eq:relationQV}
\end{equation}
then, performing the following backups:
\begin{eqnarray*}
Q(s,a) &\leftarrow& Q(s,a) + \gamma \mathcal{P}_{sa}^{s'}[V(s') - U_{sa}(s')]\\
U_{sa}(s') &\leftarrow& V(s')\thinspace,
\end{eqnarray*}
maintains this relation while updating $U_{sa}(s')$  to $V(s')$. 
\end{theorem}

\begin{theorem}
If relation (\ref{eq:relationQV}) holds and a sample $(s,a,r,s')$ is observed, then, after performing backups
\begin{eqnarray*}
N_{sa} &\leftarrow& N_{sa} + 1; \quad N_{sa}^{s'} \leftarrow N_{sa}^{s'} +1 \\
Q(s,a) &\leftarrow& \Big[Q(s,a)(N_{sa}-1) + r + \gamma U_{sa}(s')\Big]/N_{sa}\thinspace, 
\end{eqnarray*}
relation (\ref{eq:relationQV})  still holds, but with updated values for $\hat \mathcal{R}_{sa}$ and $\hat \mathcal{P}_{sa}^{s''}$.
\end{theorem}

A small action-value backup is a finer-grained version of backup (\ref{eq:QV}): performing a small backup of $Q(s,a)$ for each successor state is equivalent (in computation time complexity and effect) as performing backup (\ref{eq:QV}) once. 
While in principle, backup (\ref{eq:VQ}) can be performed after each small backup, it is not very efficient to do so, since small backups make many small changes. More efficient planning can be obtained when backup (\ref{eq:VQ}) is performed only once in a while.

In Section \ref{sec:PS small backups}, we discuss an implementation of prioritized sweeping based on small action-value backups.

\subsection{Small Backups versus Sample Backups}

A small backup has in common with a sample backup that both update a state value based on the current value of only one of the successor states. In addition, they share the same computation time complexity and their effect is in general smaller than that of a full backup.

A disadvantage of  a sample backup, with respect to a small backup, is that it introduces sampling variance, caused by a stochastic environment. This requires the use of a step-size parameter to enable averaging over successor states (and rewards). A small backup does not introduce sampling variance, since it is implicitly based on an expectation over successor states.  Hence, it does not require tuning of a step-size parameter for optimal performance.

A second disadvantage of a sample backup is that it affects the perceived distribution over action outcomes, which places some restrictions on reusing samples. For example, a model-free technique like experience replay \cite{lin:ml1992}, which stores experience samples in order to replay them at a later time, can introduce bias, which reduces performance, if some samples are replayed more often than others.  For small backups this does not hold, since the process of learning the model is independent from the backups based on the model. This allows small backups to be combined with effective selection strategies like prioritized sweeping.

\section{Prioritized Sweeping with Small Backups}
\label{sec:PS small backups}

Prioritized sweeping (PS) makes the planning step of model-based RL more efficient by using a heuristic (a `priority') for selecting backups that favours backups that are expected to cause a large value change.  A priority queue is maintained that determines which values are next in line for receiving backups.

There are two main implementations: one by \citet{moore:mlj93} and one by \citet{peng:ab1993} \footnote{We refer to the version of `queue-Dyna' for stochastic domains, which is different from the version for deterministic domains.}. All PS methods have in common that they perform backups in what we call \emph{update cycles}. By adjusting the number of update cycles that is performed per time step, the computation time per time step can be controlled. Below, we discuss in detail what occurs in an update cycle for the two classical PS implementations.

\subsection{Classical Prioritized Sweeping Implementations}

In the Moore \& Atkeson implementation the elements in the queue are states and the backups are full value backups. In control, a full value backup is different from backup (\ref{eq:full value backup}). Instead, it is equivalent (in effect and computation time) to performing backup (\ref{eq:QV}) for each action, followed by backup (\ref{eq:VQ}). Hence, the associated computation time has complexity $\mathcal{O}(|\mathcal{S}||\mathcal{A}|  + |\mathcal{A}|)$.  

An update cycle consists of the following steps. First, the top state is removed from the queue, and receives a full value backup. Let $s$ bet the top state and  $\Delta V_s$ the value change caused by the backup. Then, for all predecessor state-action pairs $(\bar s,\bar a)$ a priority $p$ is computed, using:
\begin{equation}
p  \leftarrow  \mathcal{P}_{\bar s \bar a}^{s} \cdot |\Delta V_s | \thinspace.
\end{equation}
If $\bar s$ is not yet on the queue, then it is added with priority $p$. If $\bar s$ is on the queue already, but its current priority is smaller than $p$, then the priority of $\bar s$ is upgraded to $p$.

The Peng \& Williams implementation differs from the Moore \& Atkeson implementation in that the backup is not a full value backup. Instead, it is a backup with the same effect as a small action-value backup, but with a computational complexity  of $\mathcal{O}(|\mathcal{S}| + |\mathcal{A}|)$. So, it is a cheaper backup than a full backup, but its value change is (much) smaller. The backup requires a state-action-successor triple. Hence, these triples are the elements on the queue. Predecessors are added to the queue with a priorities that estimate the action-value change.

\subsection{Small Backup implementation}

A natural small backup implementation might appear to be an implementation similar to that of Peng \& Williams, but with the main backup implemented more efficiently. The low computational cost of a small backup, however, allows for a much more powerful implementation. The pseudo-code of this implementation is shown in Algorithm \ref{al:PS small backups}. Below, we discuss some key characteristics of the algorithm.

The computation time of a small backup is so low, that it is comparable to the priority computation in the classical PS implementations. Therefore, instead of computing a priority for each predecessor and performing a backup for the element with the highest priority in the next update cycle, we can perform a small backup for all predecessors. This raises the question of what to put in the priority queue and what type of backup to perform for the top element. The natural answer is to put states in the priority queue and to perform backup (\ref{eq:VQ}) for the top state.

The priority associated with a state is based on the change in action-value that has occurred due to small backups, since the last value backup. This priority assures that states with a large discrepancy between the state value and action-values, receive a value backup first. 

One surprising aspect of the algorithm is that it does not use the function $U_{sa}$, which forms an essential part of small action-value backups. The reason is that due to the specific backup strategy used by the algorithm, $U_{sa}(s')$ is equal to $V(s')$ for all state-action pairs $(s,a)$ and all successor states $s'$. Hence, instead of using $U_{sa}$,  $V$ can be used, saving memory and simplifying the code.

Table \ref{table:computation time} shows the computation time complexity of an update cycle for the different PS implementations. The small backup implementation is the cheapest one among the three. 

\begin{table}[!tbh]
\centering
\begin{tabular}{|r|r|r|} \hline
& top-element & other  \\
& backups & \\  \hline
Moore \& Atkeson & $\mathcal{O}(|\mathcal{S}||\mathcal{A}|  + |\mathcal{A}|)$   & $\mathcal{O}(P_{re})$\\
Peng \& Williams & $\mathcal{O}(|\mathcal{S}| + |\mathcal{A}|)$    & $\mathcal{O}(P_{re})$ \\
small backups & $\mathcal{O}(|\mathcal{A}|)$ & $\mathcal{O}(P_{re})$ \\ \hline
\end{tabular}
\caption{Computation time associated with one update cycle for the different PS implementations. $P_{re}$ indicates the number of predecessors, state-action pairs that transition to the state whose value has just been updated. }
\label{table:computation time}
\end{table}

\begin{algorithm}[thb]
\begin{algorithmic}[1]
\STATE initialize $V(s)$ arbitrarily for all $s$
\STATE initialize $Q(s,a) = Q_{prev}(s,a) = V(s)$ for all $s,a$
\STATE initialize $N_{sa}, N_{sa}^{s'}$ to 0 for all $s,a,s'$
\LOOP[over episodes]
\STATE initialize $s$
\REPEAT[for each step in the episode]
\STATE select action $a$, based on $Q(s, \cdot)$
\STATE take action $a$, observe $r$ and $s'$
\STATE $N_{sa} \leftarrow N_{sa} + 1; \quad N_{sa}^{s'} \leftarrow N_{sa}^{s'} + 1$
\STATE $Q(s,a) \leftarrow \bigl[Q(s,a)(N_{sa}-1) + r  + \gamma V(s') \bigr]/N_{sa}$
\STATE $p \leftarrow |Q(s,a) -Q_{prev}(s,a)|$
\STATE if $s$ not on queue or $p > $ current priority $s$, then promote $\bar s$ to $p$ 
\FOR{ a number of update cycles}
\STATE remove top state $\bar s' $ from queue
\STATE for all $b$:  $Q_{prev}(\bar s',b) \leftarrow Q(\bar s',b)$
\STATE $tmp \leftarrow V(\bar s')$
\STATE $V(\bar s') \leftarrow \max_b Q(\bar s',b)$
\STATE $\Delta V \leftarrow  V(\bar s') - tmp$
\FORALL{$(\bar s,\bar a)$ pairs with $N_{\bar s \bar a}^{\bar s'} > 0$}
\STATE $Q(\bar s,\bar a)  \leftarrow Q(\bar s,\bar a) + \gamma  N_{\bar s\bar a}^{\bar s'}/N_{\bar s\bar a} \cdot \Delta V$
\STATE $p \leftarrow |Q(\bar s,\bar a) -Q_{prev}(\bar s,\bar a)|$
\STATE if $\bar s$ not on queue or $p > $ current priority $\bar s$, then promote $\bar s$ to $p$ 
\ENDFOR
\ENDFOR
\STATE $s \leftarrow s'$
\UNTIL{$s$ is terminal}
\ENDLOOP
\caption{Prioritized Sweeping with Small Backups}
\label{al:PS small backups}
\end{algorithmic}
\end{algorithm}

\section{Experimental Results}

In this section, we evaluate the performance of a minimal version of Algorithm \ref{al:small backup evaluation}, as well as the performance of Algorithm \ref{al:PS small backups}.

\subsection{Small backup versus Sample backup}
\label{section:small backup experiments}

We compare the performance of TD(0), which performs one sample backup per time step, with a version of Algorithm 1 that performs one small backup per time step. Specifically, its number of iterations (line 8) is 1, and the selected state-successor pair (line 9) is the pair corresponding to the most recent transition.

Their performance is compared on two evaluation tasks, both consisting of 10 states, laid out in a circle. State transitions only occur between neighbours.  The transition probabilities for both tasks are generated by a random process. Specifically, the transition probability to a neighbour state is generated by a random number between 0 and 1 and normalized such that the sum of the transition probabilities to the left and right neighbour is 1. The reward for counter-clockwise transitions is always +1. The reward for clockwise transitions is different for the two tasks. In the first task, a clockwise transition results in a reward of -1; in the second task, it results in a reward of +1. The discount factor $\gamma$ is 0.95 and the initial state values are 0.

For TD(0), we performed experiments with a constant step-size for values between 0 and 1 with steps of 0.02. In addition, we performed experiments with a decaying, state-dependent step-size, according to
\begin{equation}
\alpha(s) =   \frac{1}{d\cdot (N_s - 1) + 1}\quad,
\end{equation}
where $N_s$ is the number of times state $s$ has been visited, and $d$ specifies the \emph{decay rate}.  We used values of $d$ between 0 and 1 with steps of 0.02. Note that for $d=0$, $\alpha(s) = 1$, and for $d=1$, $\alpha(s) = 1 / N_s$.

Each time a transition is observed and the corresponding backup is performed, the root-mean squared (RMS) error over all states is determined. The average RMS error over the first 10.000 transitions, normalized with the initial error, determines the performance.  Figure \ref{fig:eval_results} shows this performance, averaged over 100 runs. The standard error is negligible: the maximum standard error in the first task was 0.0057 (after normalization) and in the second task 0.0007. Note that the performance for $d = 0$ is equal to the performance for $\alpha = 1$, as it should, by definition. The normalized performance for $\alpha = 0$ is 1, since no learning occurs in this case.

These experiments demonstrate three things. First, the optimal step-size can vary a lot between different tasks. Second, selecting a sub-optimal step-size can cause large performance drops. Third, a small-backup, which is parameter-free, has a performance similar to the performance of TD(0) with optimized step-size. Since the computational complexity is the same, the small backup is a very interesting alternative to the sample backup in domains with tight constraints on the computation time,  where previously only sample backups where viable. Keep in mind that a sample backup does require a model estimate, so if there are also tight constraints on the memory, a sample backup might still be the only option.

\begin{figure}[thb]
\begin{center}
\includegraphics[width=\columnwidth]{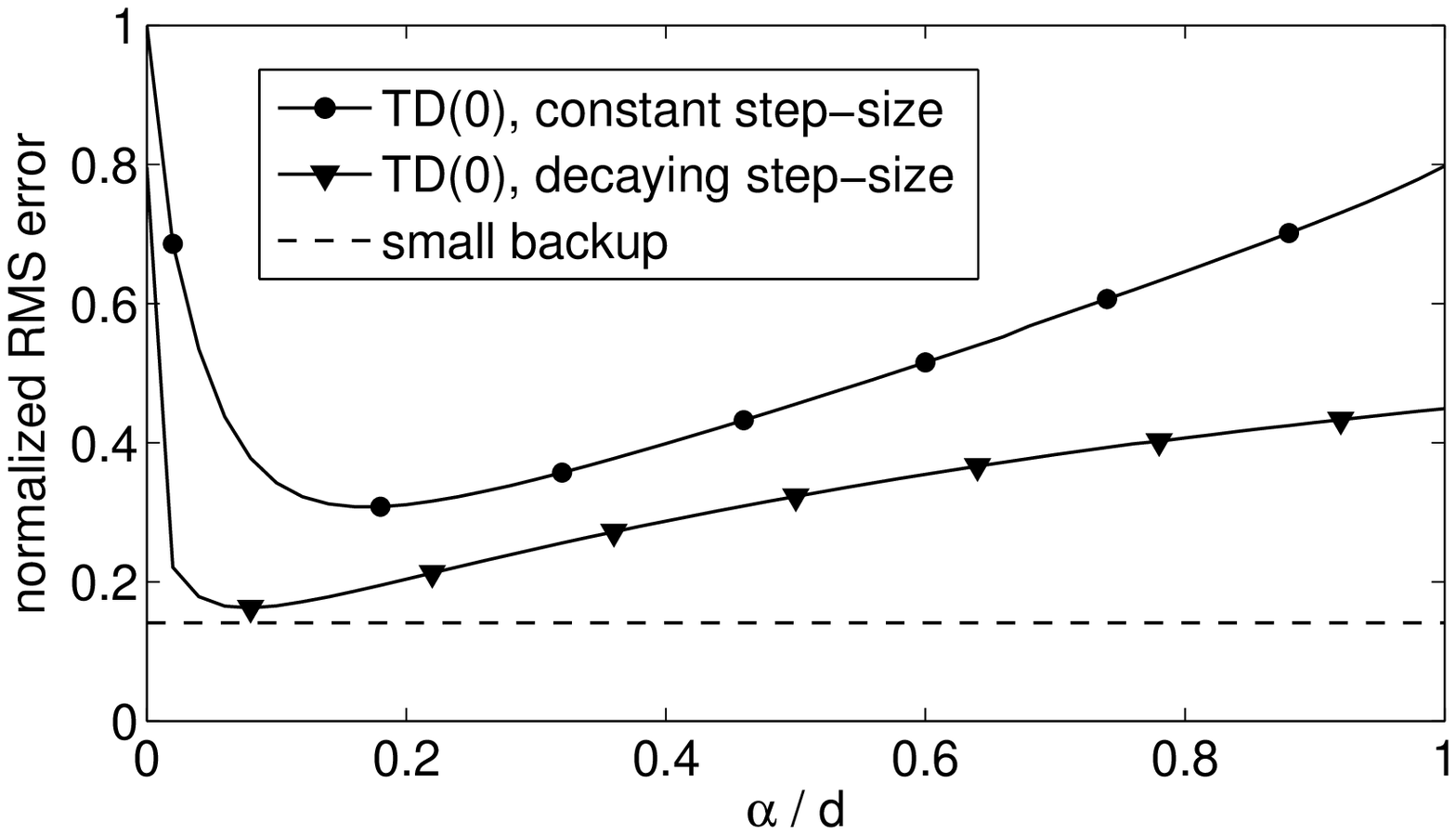}
\vspace{0cm}
\includegraphics[width=\columnwidth]{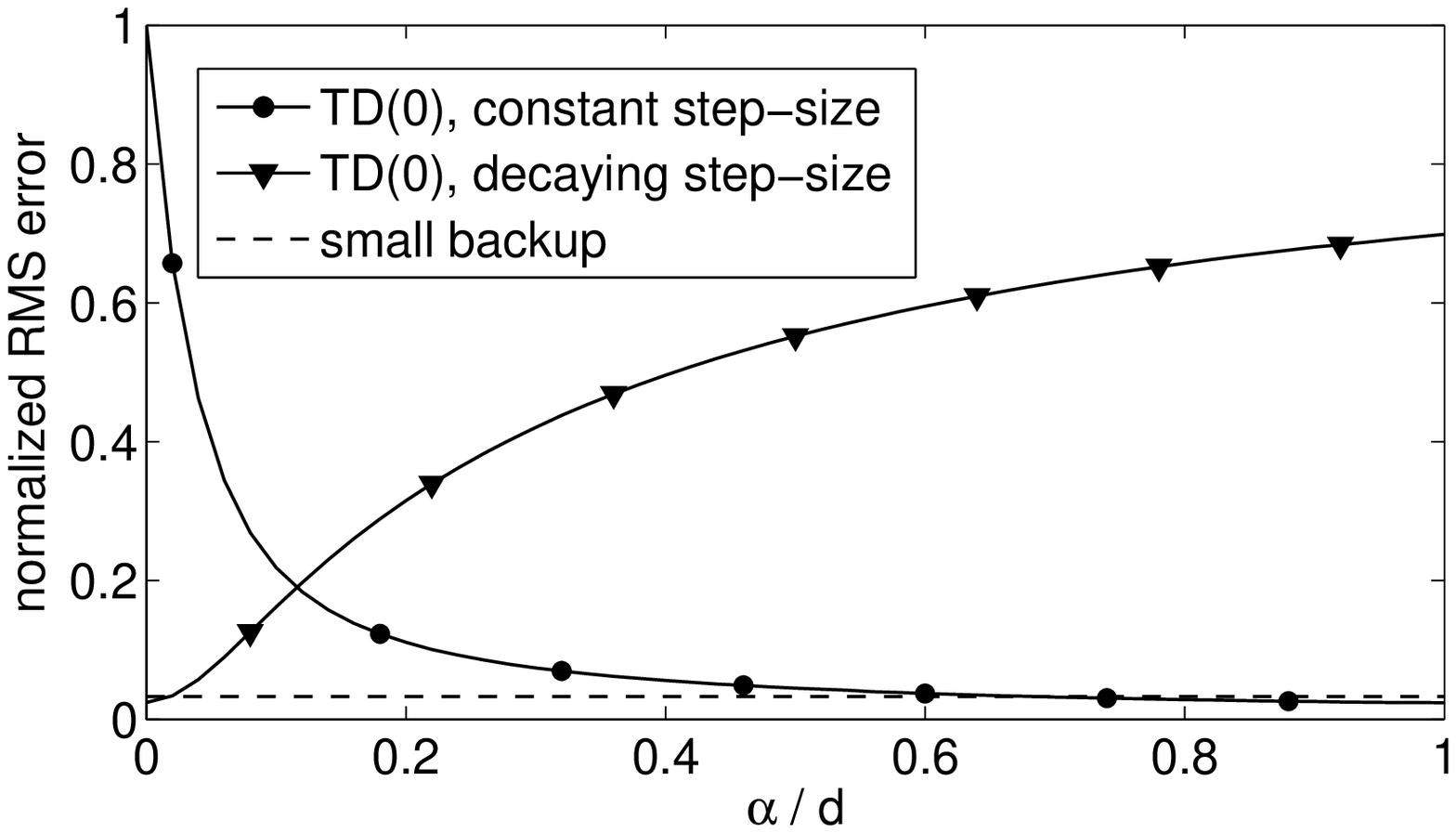}
\caption{Average RMS error over the first 10.000 observations, normalized by the initial error, for different values of the step-size parameter $\alpha$, in case of constant step-size, or different values of the decay parameter $d$, in case of decaying step-size. The top graph corresponds with the first evaluation task; the bottom graph with the second.}
\label{fig:eval_results}
\end{center}
\end{figure}

\subsection{Prioritized Sweeping}

We compare the performance of prioritized sweeping with small backups (Algorithm \ref{al:PS small backups}) with the two classical implementations of Moore\&Atkeson and Peng\&Williams on the maze task depicted in the top of Figure \ref{fig:LargeMaze2}. The reward received at each time step is -1 and the discount factor is 0.99. The agent can take four actions, corresponding to the four compass directions, which stochastically move the agent to a different square. The bottom  of Figure \ref{fig:LargeMaze2} shows the relative action outcomes of a `north' action. In free space, an action can result in 15 possible successor states, each with equal probability. When the agent is close to a wall, this number decreases.

To obtain an upper bound on the performance, we also compared against a method that performs  value iteration (until convergence) at each time step, using the most recent model estimate.

As exploration strategy, the agent select with 5\% probability a random action, instead of the greedy one. On top of that, we use the `optimism  in the face of uncertainty' principle, as also used by Moore \& Atkeson. This means that as long as a state-action pair has not been visited for at least M times, it's value is defined as some optimistically defined value (0 for our maze task), instead of the value based on the model estimate. We optimized $M$ for the value iteration method, resulting in $M=4$, and used this value for all methods.

We performed experiments for 1, 3, 5 and 10 update cycles per time step. Figure \ref{fig:PS_results} shows the average return over the first 200 episodes for the different methods. The results are averaged over 100 runs. The maximum standard deviation is 0.1 for all methods, except for the method of Peng \& Williams, which had a maximum standard deviation of 1.0.

The computation time per update cycle was about the same for the three different PS implementations, with a small advantage for the small backup implementation, which shows that the $\mathcal{O}(P_{re})$ computation (see Table \ref{table:computation time}) is dominant in this task. The computation time per observation of the value iteration method was more than 400 times as high as a single update cycle.

PS with small backups turns out to be very effective. With only a single update cycle, the value-iteration result can be closely approximated, in contrast to the two classical implementations. The results also show that the Peng \& Williams method performs considerably worse than the one of Moore \& Atkeson in the considered domain. This can be explained by the different backups they perform. The effect of the backup of Peng \& Williams is proportional to the transition probability, which in most cases is $\frac{1}{15}$. In contrast, the Moore \&  Atkeson method performs a full backup each update cycle. While the small backup implementation also uses backups that are proportional to the transition probability, it performs a lot more backups per update cycle. Specifically, a number that is proportional to the number of predecessors. In general, this number will increase when the stochasticity of the domain increases.

\begin{figure}[thb]
\begin{center}
\includegraphics[height=5.5 cm]{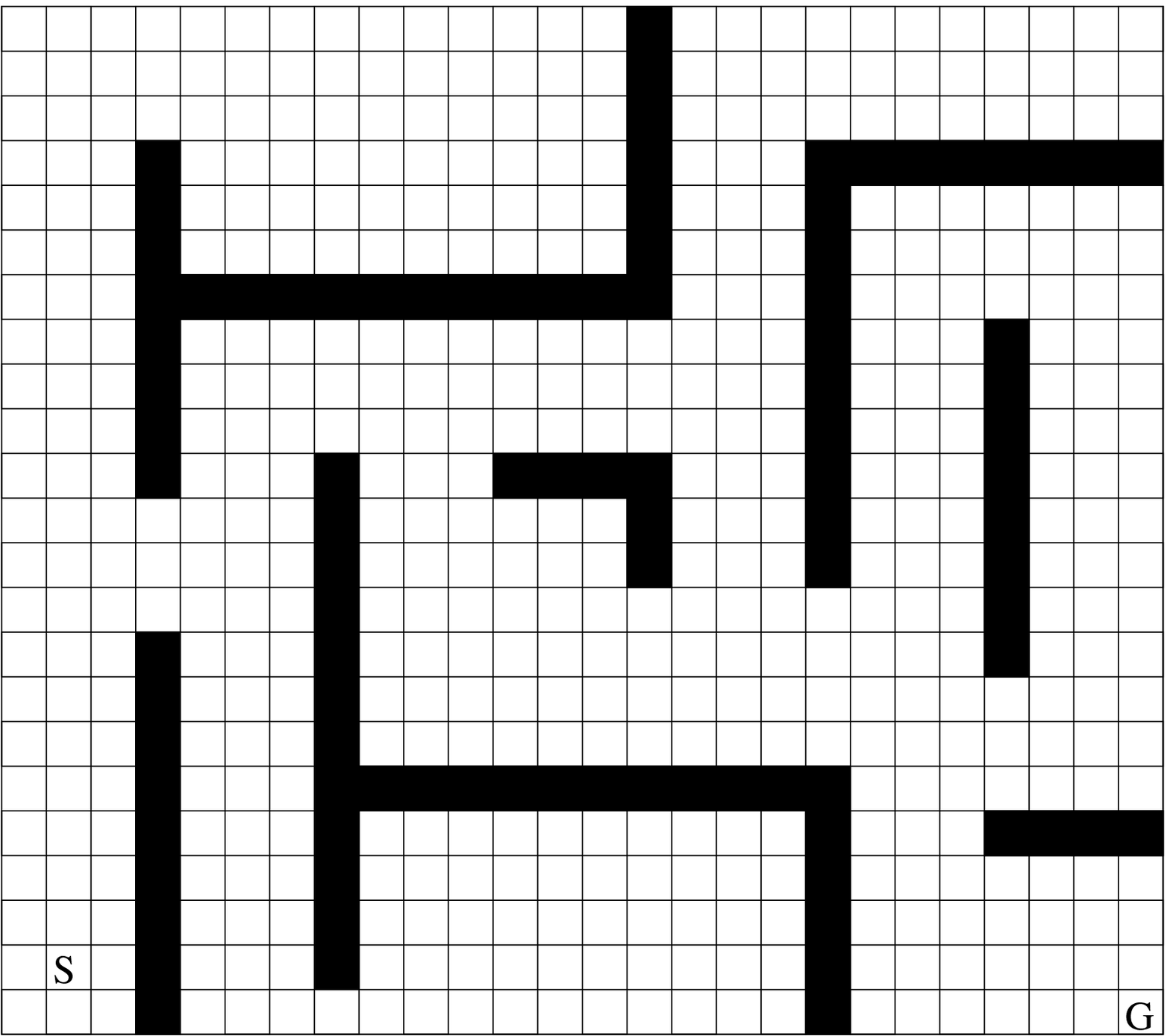}
\vspace{0cm}
\includegraphics[width=6 cm]{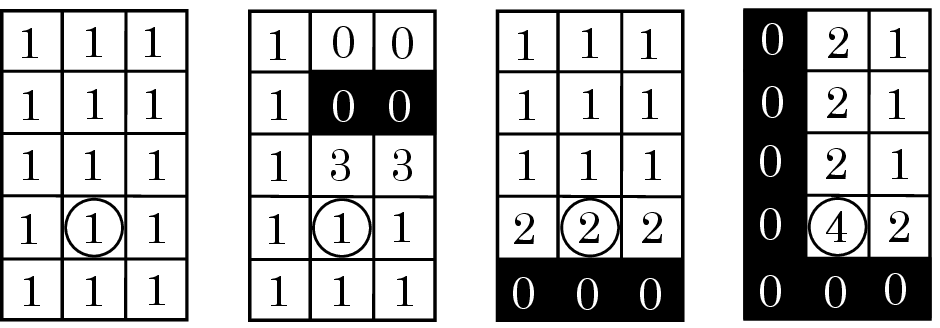}
\caption{Above, the maze task, in which the agent must travel from $S$ tothe $G$.  Below, transition probabilities ($\cdot \frac{1}{15}$) of a `north' action for different positions of the agent (indicated by the circle) with respect to the walls (black squares).}
\label{fig:LargeMaze2}
\end{center}
\end{figure}

\begin{figure}[thb]
\begin{center}
\includegraphics[width=\columnwidth]{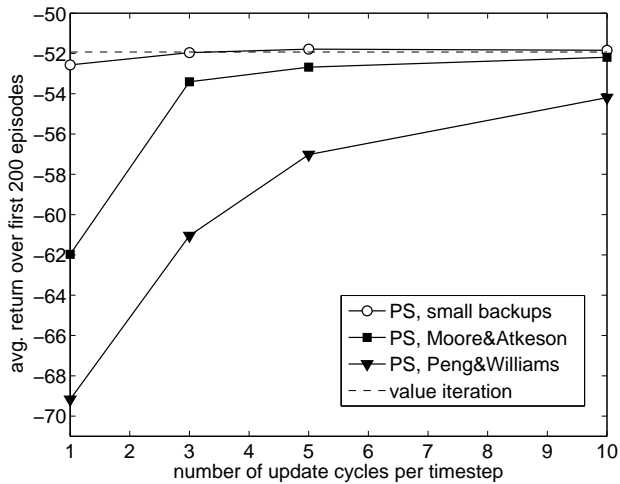}
\caption{Performance of the PS implementations on the maze task for a different number of update cycles per time step and a method that performs value iteration at each time step.}
\label{fig:PS_results}
\end{center}
\end{figure}

\section{Discussion}

Prioritized sweeping can be viewed as a generalization of the idea of replaying of experience in backward order \cite{lin:ml1992}, which by itself is related to \emph{eligibility traces} \citep{sutton:ml1988, watkins89, sutton:wals94}. What all these techniques have in common is that new information (which can be value changes, but at its core all value changes originate from new data) is propagated backwards. Whereas backward replay and eligibility traces use the recent trajectory for backward propagation of information, prioritized sweeping uses a model estimate for this. Hence, it propagates new information more broadly.

What gives the performance edge to the small backup implementation is that it implements the principle of backward updating in a cleaner and more efficient way. One update cycle of Algorithm \ref{al:PS small backups} represents, in a way, the ultimate backwards backup:  all predecessors are updated with the current value of a chosen state, which is selected because it recently experienced a large value change. In contrast, the other PS implementation place the predecessors in a queue and backup only the state with the highest priority in the next update cycle. On top of that, the computation time per update cycle is lower for the small backup implementation (see Table \ref{table:computation time}).

The new implementation of PS introduced in this paper would be impossible without the new backup. The small backup allows for very targeted updates that are computationally very cheap. This enables finer control over how computation time is spend, which is what drives the new PS implementation.

\section{Conclusion}

We demonstrated in this paper that the planning step in model-based reinforcement learning method can be done substantially more efficient by making use of small backups. These backups are finer-grained version of a full backup, which allow for more control over how the available computation time is spend. This makes new, more efficient, update strategies possible. In addition, small backups can be useful in domains with very tight time constraints, offering a parameter-free alternative to sample backups, which were up to now often the only feasible option for such domains.


\begin{thebibliography}{8}
\providecommand{\natexlab}[1]{#1}
\providecommand{\url}[1]{\texttt{#1}}
\expandafter\ifx\csname urlstyle\endcsname\relax
  \providecommand{\doi}[1]{doi: #1}\else
  \providecommand{\doi}{doi: \begingroup \urlstyle{rm}\Url}\fi

\bibitem[Kaelbling et~al.(1996)Kaelbling, Littman, and Moore]{kaelbling:jair96}
Kaelbling, L. P., Littman, M. L., Moore, A. P. (1996).
\newblock Reinforcement learning: A survey.
\newblock \emph{Journal of Artificial Intelligence Research}, 4:\penalty0
  237--285.

\bibitem[Lin(1992)]{lin:ml1992}
Lin, L.J. (1992).
\newblock Self-improving reactive agents based on reinforcement learning,
  planning and teaching.
\newblock \emph{Machine learning}, 8\penalty0 (3):\penalty0 293--321.

\bibitem[Moore \& Atkeson(1993)Moore and Atkeson]{moore:mlj93}
Moore, A., Atkeson, C. (1993).
\newblock Prioritized sweeping: Reinforcement learning with less data and less
  real time.
\newblock \emph{Machine Learning}, 13:\penalty0 103--130.

\bibitem[Peng \& Williams(1993)Peng and Williams]{peng:ab1993}
Peng, J., Williams, R. J. (1993).
\newblock Efficient learning and planning within the dyna framework.
\newblock \emph{Adaptive Behavior}, 1\penalty0 (4):\penalty0 437--454.

\bibitem[Sutton(1988)]{sutton:ml1988}
Sutton, R. S. (1988).
\newblock Learning to predict by the methods of temporal differences.
\newblock \emph{Machine learning}, 3\penalty0 (1):\penalty0 9--44.

\bibitem[Sutton \& Barto(1998)Sutton and Barto]{sutton:book98}
Sutton, R. S., Barto, A. G. (1998).
\newblock \emph{Reinforcement Learning: An Introduction}.
\newblock MIT Press, Cambridge, Massachussets.

\bibitem[Sutton \& Singh(1994)Sutton and Singh]{sutton:wals94}
Sutton, R. S., Singh, S. P. (1994).
\newblock On step-size and bias in temporal-difference learning.
\newblock In \emph{Proceedings of the Eight Yale Workshop on Adaptive and
  Learning Systems}.

\bibitem[Watkins(1989)]{watkins89}
Watkins, C. (1989).
\newblock \emph{{Learning from delayed rewards}}.
\newblock PhD thesis, King's College, Cambridge, England.

\end{thebibliography}
\end{document}